\documentclass[letterpaper, 10 pt, conference]{ieeeconf}  

\IEEEoverridecommandlockouts                              
\usepackage{amsmath,amssymb,mathtools}
\usepackage{bm}

\usepackage{graphicx}
\usepackage{booktabs}
\usepackage{multirow}
\usepackage{array}
\allowdisplaybreaks

\usepackage{algorithm}
\usepackage{algpseudocode}
\usepackage{xcolor}
\usepackage{hyperref}
\usepackage{cleveref}
\usepackage{microtype}
\usepackage{caption}
\usepackage{subcaption}
\usepackage{wrapfig}
\usepackage{tikz}

\hypersetup{
     colorlinks=true,
     linkcolor=black,
     filecolor=blue,
     citecolor = blue,      
     urlcolor=blue,
     }

\newtheorem{theorem}{Theorem}[section]

\newtheorem{proposition}[theorem]{Proposition}
\newtheorem{corollary}[theorem]{Corollary}

\newtheorem{assumption}{Assumption}

\newtheorem{remark}[theorem]{Remark}

\newcommand{\RR}{\mathbb{R}}
\newcommand{\EE}{\mathbb{E}}

\newcommand{\cS}{\mathcal{S}}
\newcommand{\cA}{\mathcal{A}}
\newcommand{\cF}{\mathcal{F}}
\newcommand{\cM}{\mathcal{M}}

\newcommand{\Jr}{J_r}
\newcommand{\Jc}{J_c}
\newcommand{\gr}{\bm{g}_r}
\newcommand{\gc}{\bm{g}_c}
\newcommand{\gperp}{\bm{g}_r^{\perp}}
\newcommand{\gb}{\bm{g}_b}
\newcommand{\gup}{\bm{g}_{\mathrm{up}}}
\newcommand{\pith}{\pi_{\bm{\theta}}}
\newcommand{\pithstar}{\pi_{\bm{\theta}^*}}
\newcommand{\nablath}{\nabla_{\bm{\theta}}}
\newcommand{\dhat}{\hat{\delta}}
\newcommand{\lamt}{\lambda_t}
\newcommand{\lamstar}{\lambda^*}
\newcommand{\norm}[1]{\left\|#1\right\|}
\newcommand{\inner}[2]{\left\langle #1,\, #2 \right\rangle}

\newcommand{\clip}{\mathrm{clip}}

\definecolor{myblue}{RGB}{31,56,100}
\definecolor{mygreen}{RGB}{0,100,60}
\definecolor{mygray}{RGB}{120,120,120}
\definecolor{proofbg}{RGB}{248,248,252}

\title{\LARGE \bf
Boundary-Seeking Policy Gradient for Safe Reinforcement Learning
}

\author{Chenhua Fan$^{1\dagger}$, Jiahui Zhu$^{1\dagger}$, Yuhang Zhang$^{1}$, and Honghao Wei$^{1*}$
\thanks{$^\dagger$ contributed equally, $^*$ corresponding author}
\thanks{$^1$School of EECS, Washington State University, Pullman, WA 99163, USA
       {\tt\small (E-mails: \{chenhua.fan,jiahui.zhu, yuhang.zhang1, honghao.wei\}@wsu.edu)}}
}

\begin{document}

\maketitle
\thispagestyle{empty}
\pagestyle{empty}

\begin{abstract}
Safe reinforcement learning maximizes reward subject to safety constraints. For Constrained Markov Decision Processes, the linear-programming view over occupancy measures implies that whenever the constraint is active at optimality, the optimal policy lies exactly on the constraint boundary, yet standard gradient-based methods do not exploit this structure and often settle in the feasible interior. We introduce Boundary-Seeking Policy Gradient (BSPG), a first-order method whose update combines a tangential component that improves reward while preserving cost to first order with a signed, residual-driven normal component that regulates the policy toward the active boundary from either side; the combined direction admits an algebraic Lagrangian form with an induced coefficient and no learned dual variable. Under exact gradients and stated regularity conditions, the constraint residual converges to zero from either side with a finite-horizon $O(1/\sqrt{T})$ bound, the tangential component is a reward-ascent direction on the boundary, and any convergent parameter sequence is stationary on the active constraint set, satisfying the KKT conditions when the limit is also a local maximizer over the feasible set. This complements existing analyses, which certify feasibility but do not characterize the constraint value at convergence. On a standard Safety-Gymnasium navigation task, BSPG attains higher reward while tracking the boundary more tightly than the compared baselines.
\end{abstract}

\section{Introduction}
Reinforcement learning (RL) has achieved remarkable success in sequential decision-making tasks including game playing \cite{DavThoJul_18}, autonomous driving \cite{LiPenFen_22}, and robotics \cite{KobBagPet_13}. In safety-critical applications, however, maximizing reward alone is insufficient: a deployed policy must also satisfy operational constraints, a requirement commonly formalized through Constrained Markov Decision Processes (CMDPs) \cite{Alt_99}.

Existing constrained policy optimization follows several directions. \emph{Primal--dual methods} \cite{StoAchAbb_20,PatCalMig_19,WeiLiuYin_22,WeiGhoShr_23,ZhaPenWei_24,DinWeiYan_20,ZhuYuLee_25} enforce the constraint through a learned multiplier, so the policy's location relative to an active boundary is controlled only indirectly through the dual dynamics. \emph{Primal and projection-based methods} \cite{AchHelDav_17,YanRosNar_20,ZhaVuoRos_20,YanJiDai_22}, and switching schemes such as CRPO \cite{XuLiaLan_21}, restrict or correct updates by feasibility --- a one-sided condition that does not say how much of the remaining cost budget a feasible policy should use. \emph{Gradient-manipulation methods} \cite{GuSelDin_24,GuShiDin_24,YaoLiuCen_24} coordinate the two objectives through gradient geometry, which captures local alignment but not where the policy sits relative to the cost limit. In short, standard updates do not separate reward improvement along a cost level set from signed regulation across cost levels toward an active boundary.

This separation matters because of a well-known structural property: when the CMDP model is known, safe RL admits an exact linear-programming solution over occupancy measures \cite{Alt_99}, and whenever the constraint is active at optimality, the optimal policy lies on the boundary of the feasible set. A strictly feasible policy therefore retains cost slack that can be locally converted into reward when the reward and cost gradients are positively aligned; Section~\ref{sec:boundary} makes both the global statement and this local trade-off precise.

Motivated by this structure, we propose Boundary-Seeking Policy Gradient (BSPG), which makes signed boundary regulation an explicit, separate component of the update. A tangential term removes from the reward gradient its component along the cost gradient, preserving cost to first order while retaining reward ascent; a normal term proportional to the signed residual decreases cost when the policy is infeasible and increases cost toward the limit when it is strictly feasible. The ideal update admits an algebraic Lagrangian representation with an induced coefficient and maintains no learned dual variable. Our contributions are:
\begin{enumerate}
      \item  We characterize the boundary structure of constrained policy optimization: when no unconstrained reward maximizer is feasible, every constrained occupancy-measure maximizer uses the full cost budget, this target transfers to any policy class attaining the exact constrained optimum, and we identify when local slack supports first-order reward improvement.

    \item  We propose BSPG and, for the ideal exact-gradient update, prove two-sided residual convergence with a finite-horizon $O(1/\sqrt{T})$ bound, reward ascent of the tangential component, and a characterization of any convergent run: the limit is stationary on the active set and satisfies the KKT conditions when it is also a local maximizer over the feasible set.
\item  On a standard Safety-Gymnasium navigation task, BSPG achieves higher reward with tighter boundary adherence than the compared baselines.
 \vspace{-1ex}
\end{enumerate}

\begin{figure*}[t]
    \centering
    \includegraphics[width=0.80\textwidth]{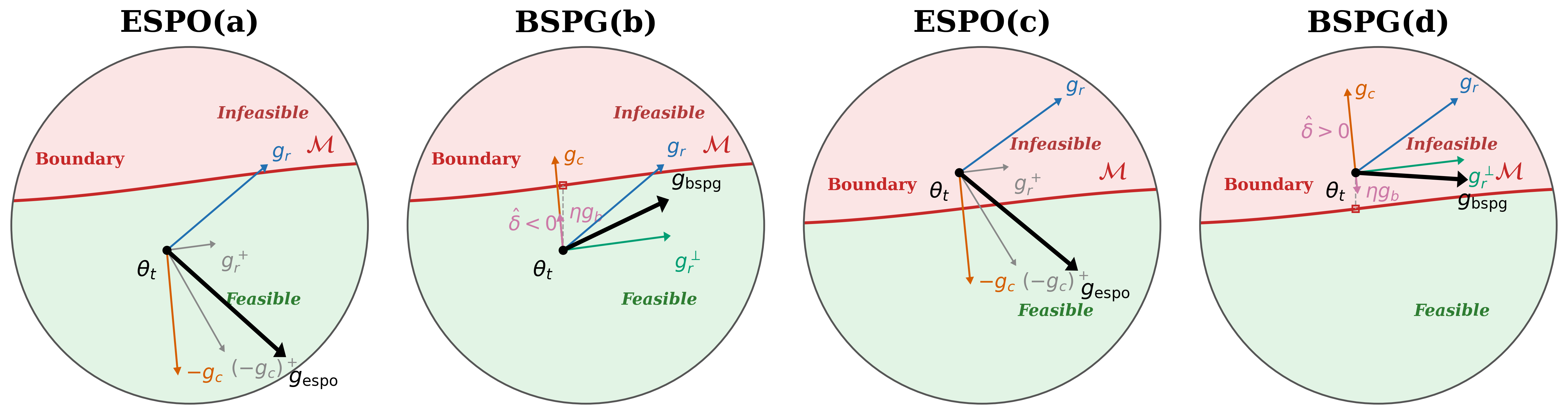}
    \caption{Local update geometry of ESPO~\cite{GuShiDin_24} and \textbf{BSPG}. ESPO shapes its direction from the reward--cost gradient relation; BSPG adds a normal component set by the signed constraint residual. For the ideal update ($\hat{\delta}=\delta$), this component points toward $\cM$ from both the feasible ($\hat{\delta}<0$) and infeasible ($\hat{\delta}>0$) sides.}
    \label{fig:grad_update}
    \vspace{-2ex}
\end{figure*}

\section{Related Work}
\label{sec:related}

\textbf{Primal--dual methods} enforce safety through a Lagrange multiplier balancing reward against violation. PPO- and TRPO-Lagrangian \cite{RayAchAmo_19} update the multiplier from the observed violation, with coupled dynamics that can be sensitive and oscillatory; PID-Lagrangian \cite{StoAchAbb_20} stabilizes the dual update at the cost of extra tuning, and RCPO \cite{TesManMan_18} and CAL \cite{WuTanLin_24} refine cost estimation and the coupled updates. In all cases the residual acts on the multiplier, which then reshapes the policy step: boundary regulation is not an explicit component of the update itself.

\textbf{Primal methods} restrict or correct updates by feasibility. CPO \cite{AchHelDav_17} solves a constrained local subproblem with guarantees but added computational cost \cite{RayAchAmo_19}; PCPO \cite{YanRosNar_20}, FOCOPS \cite{ZhaVuoRos_20}, and CUP \cite{YanJiDai_22} give first-order and projection-based variants; CRPO \cite{XuLiaLan_21} alternates reward and constraint steps by feasibility status, and CPPO \cite{XuaZhaYin_23} adds an infeasibility-recovery mechanism. Feasibility, however, is one-sided: a policy with large slack and one at the boundary satisfy the same inequality, and when slack is locally convertible into reward (Section~\ref{sec:boundary}), feasibility alone does not determine how much of the budget to use.

\textbf{Gradient-manipulation methods} are closest to our geometric view. PCRPO \cite{GuSelDin_24} selects or combines directions by the reward--cost gradient angle, GradS \cite{YaoLiuCen_24} extends this to multiple constraints, and ESPO \cite{GuShiDin_24} uses gradient interaction to improve efficiency. But gradient geometry is not boundary location: two policies with similar gradient directions can hold very different slack, so geometry alone determines neither the direction nor the magnitude of motion across cost levels. These methods can reach an active boundary, yet their updates do not use the signed residual as a separate normal correction regulating the policy toward it.

\section{Preliminaries}

We consider a CMDP defined by the tuple
$(\cS, \cA, P, r, c, \gamma, \mu_0, d)$, where $\cS$ and $\cA$ denote the state and action spaces, $P(s' \mid s,a)$ is the transition kernel, $r:\cS \times \cA \to \RR$ and $c:\cS \times \cA \to \RR$ are the reward and cost functions, $\gamma \in (0,1)$ is the discount factor, $\mu_0$ is the initial-state distribution, and $d \in \RR$ is the cost threshold. A policy $\pi:\cS \to \Delta(\cA)$ maps each state to a distribution over actions. For a policy $\pi_{\bm\theta}$ parameterized by $\bm\theta \in \RR^p$, define the expected discounted reward and cost as
\begin{align}
    \Jr(\pi_{\bm\theta})
    &= \EE_{\tau \sim \pi_{\bm\theta}}
    \!\left[\sum_{t=0}^{\infty}\gamma^t r(s_t,a_t)\right], \label{eq:Jr}\\
    \Jc(\pi_{\bm\theta})
    &= \EE_{\tau \sim \pi_{\bm\theta}}
    \!\left[\sum_{t=0}^{\infty}\gamma^t c(s_t,a_t)\right]. \label{eq:Jc}
\end{align}
Safe RL aims to solve the following constrained optimization problem:
\begin{equation}
    \max_{\bm\theta}\; \Jr(\pi_{\bm\theta})
    \;\;
    \textit{s.t.}
    \;\;
    \Jc(\pi_{\bm\theta}) \le d.
    \label{eq:cmdp}
\end{equation}
To describe the constraint geometry in policy parameter space, define
\begin{align}
  &\delta(\bm\theta)
  \triangleq
  \Jc(\pi_{\bm\theta})-d,
  &&\text{(constraint residual)},
  \label{eq:delta}\\
 & \cF
  \triangleq
  \{\bm\theta\in\RR^p:\delta(\bm\theta)\le0\},
  &&\text{(feasible parameter set)},
  \label{eq:feasible}\\
&  \cM
  \triangleq
  \{\bm\theta\in\RR^p:\delta(\bm\theta)=0\},
  &&\text{(active constraint set)}.
  \label{eq:manifold}
\end{align}
Since our method is a policy-gradient algorithm, we next introduce the reward and cost gradients that determine the local update direction. These quantities also provide the basis for the geometric decomposition developed later.
By the policy gradient theorem \cite{RicDavSat_99}, define the normalized discounted state distribution $d_\gamma^{\pi_{\bm\theta}}(s)=(1-\gamma)\sum_{t=0}^{\infty}\gamma^t\Pr_{\pi_{\bm\theta}}(s_t=s)$. Then
\begin{align}
    \gr
    &= \frac{1}{1-\gamma}\EE_{s \sim d_\gamma^{\pi_{\bm\theta}},\, a \sim \pi_{\bm\theta}}
    \!\left[
        A_r^{\pi_{\bm\theta}}(s,a)\,\nabla_{\bm\theta}\log \pi_{\bm\theta}(a\mid s)
    \right], \label{eq:gr}\\
    \gc
    &= \frac{1}{1-\gamma}\EE_{s \sim d_\gamma^{\pi_{\bm\theta}},\, a \sim \pi_{\bm\theta}}
    \!\left[
        A_c^{\pi_{\bm\theta}}(s,a)\,\nabla_{\bm\theta}\log \pi_{\bm\theta}(a\mid s)
    \right], \label{eq:gc}
\end{align}
where $A_r^{\pi_{\bm\theta}}(s,a)=Q_r^{\pi_{\bm\theta}}(s,a)-V_r^{\pi_{\bm\theta}}(s)$ is the reward advantage, and
$A_c^{\pi_{\bm\theta}}(s,a)=Q_c^{\pi_{\bm\theta}}(s,a)-V_c^{\pi_{\bm\theta}}(s)$ is the cost advantage.

The key observation during policy improvements is that the reward and cost gradients may not be aligned. For a local update $\bm\theta'=\bm\theta+\alpha\bm v$, differentiability gives $J_k(\bm\theta')-J_k(\bm\theta)=\alpha\inner{\bm g_k}{\bm v}+o(\alpha)$, so $\bm v$ improves reward to first order if $\inner{\gr}{\bm v}>0$ and reduces cost if $\inner{\gc}{\bm v}<0$. When $\inner{\gr}{\gc}>0$, a naive ascent step along $\gr$ increases both reward and cost to first order, the \textbf{gradient conflict} near an active constraint that our method is designed to handle.

\section{Boundary Optimality Principle}
\label{sec:boundary}
We begin with a basic structural observation about the CMDP in~\eqref{eq:cmdp} that motivates our algorithm design. When the constraint is active at the optimum, the optimal constrained solution should lie on the boundary of the feasible set rather than in its strict interior. Intuitively, if a feasible policy $\pi_{\bm{\theta}}$ satisfies the constraint with strict inequality, i.e., $\Jc(\pi_{\bm{\theta}}) < d$, then, by continuity of $\Jc$ in $\bm{\theta}$, there exists a neighborhood around $\bm{\theta}$ in which the constraint remains inactive. In such a region, moving toward the feasible boundary can potentially improve reward without immediately violating the constraint. The following theorem formalizes this observation.

\begin{theorem}[Boundary Optimality]
\label{thm:optimality}
Consider a finite discounted CMDP\eqref{eq:cmdp} with finite state and action spaces over the class $\Pi$ of all stationary policies. Suppose the feasible set $\{\pi\in\Pi:\Jc(\pi)\le d\}$ is non-empty and that no unconstrained maximizer of $\Jr$ over $\Pi$ is feasible, i.e., every $\pi_r^\star\in\arg\max_{\pi\in\Pi}\Jr(\pi)$ satisfies $\Jc(\pi_r^\star)>d$. Then every constrained-optimal policy $\pi^\star\in\arg\max_{\pi\in\Pi:\,\Jc(\pi)\le d}\Jr(\pi)$ satisfies $\Jc(\pi^\star)=d$.
\end{theorem}
\noindent\emph{Proof.} For a stationary policy $\pi$, define its (normalized) occupancy measure
$\rho_\pi(s,a)=(1-\gamma)\sum_{t=0}^\infty \gamma^t \Pr_{\pi,\mu_0}(s_t=s,a_t=a)$.
The set $\Lambda=\{\rho_\pi:\pi\in\Pi\}$ is a compact convex polytope, and
$\Jr,\Jc$ are linear functionals of $\rho$:
$J_k(\pi)=\frac{1}{1-\gamma}\sum_{s,a}\rho_\pi(s,a)\,k(s,a)$ for
$k\in\{r,c\}$ \cite{Alt_99}; hence the maxima below are attained.

Suppose, for contradiction, that some constrained-optimal $\pi^\star$
satisfies $\Jc(\pi^\star)<d$, and write $\rho^\star=\rho_{\pi^\star}$.
Fix any unconstrained maximizer $\pi_r^\star$ with occupancy measure
$\rho_r$; by hypothesis $\Jc(\rho_r)>d$. If $\Jr(\rho_r)=\Jr(\rho^\star)$,
then $\pi^\star$ would itself be a feasible unconstrained maximizer,
contradicting the hypothesis; hence $\Jr(\rho_r)>\Jr(\rho^\star)$.
For $\lambda\in(0,1]$ let $\rho_\lambda=(1-\lambda)\rho^\star+\lambda\rho_r
\in\Lambda$ (convexity). By linearity,
$\Jr(\rho_\lambda)=(1-\lambda)\Jr(\rho^\star)+\lambda\Jr(\rho_r)
>\Jr(\rho^\star)$, and
$\Jc(\rho_\lambda)=(1-\lambda)\Jc(\rho^\star)+\lambda\Jc(\rho_r)\le d$
for every
$\lambda\le\bigl(d-\Jc(\rho^\star)\bigr)/\bigl(\Jc(\rho_r)-\Jc(\rho^\star)\bigr)$,
a strictly positive threshold. Thus, for small $\lambda>0$, the policy
induced by $\rho_\lambda$ is feasible with strictly larger reward than
$\pi^\star,$ a contradiction. Hence $\Jc(\pi^\star)=d$.
\hfill$\blacksquare$

The proof, which relies only on the convexity of the occupancy-measure polytope and the linearity of $\Jr,\Jc$ in the occupancy measure \cite{Alt_99}.

\begin{corollary}[Boundary optimality under realizability]
\label{cor:realizability}
Suppose the assumptions of Theorem~\ref{thm:optimality} hold, the maximum below is attained, and
\begin{equation}
  \max_{\bm{\theta}\in\cF}\Jr(\pith)
  \;=\;
  \max_{\pi\in\Pi:\,\Jc(\pi)\le d}\Jr(\pi).
  \label{eq:realizability}
\end{equation}
Then every global maximizer of the parameterized problem \eqref{eq:cmdp} belongs to $\cM$.
\end{corollary} Indeed, any global maximizer $\bm{\theta}^\star$ attains the constrained value, so the induced policy $\pi_{\bm{\theta}^\star}$ is constrained-optimal in $\Pi$, and Theorem~\ref{thm:optimality} gives $\Jc(\pi_{\bm{\theta}^\star})=d$. Thus, nonconvex parameterization may create additional local stationary points, but it does not change the active cost level of a realizable global optimum.

Theorem~\ref{thm:optimality} provides the basic theoretical motivation for BSPG: when the constraint is active at the optimum, the constrained optimum of the underlying CMDP uses the full cost budget. Thus, in safe RL, the objective is not merely to remain feasible, but to approach and optimize along the active constraint boundary, where the optimal constrained solution is attained.

\begin{remark}[Occupancy-measure vs.\ parameter space]
\label{rem:occ-param}
Theorem~\ref{thm:optimality} is a statement about the CMDP over all stationary policies: in occupancy-measure space the feasible set is a compact convex polytope and both objectives are linear, so the boundary property follows from convexity. In the parameter space used in \eqref{eq:delta}--\eqref{eq:manifold}, $\Jr(\pith)$ and $\Jc(\pith)$ are in general nonconcave, and the global property transfers only when the policy class can represent (near-)optimal policies, e.g., tabular softmax or sufficiently expressive parameterizations. Accordingly, our algorithm uses this geometry only locally: $\cM$ serves as the target set, and the guarantees of Section~\ref{sec:theory} are first-order statements about the parameter-space iterates under Assumptions~\ref{asm:smooth}--\ref{asm:stepsize}; they do not assert that every parameter-space constrained maximizer lies on $\cM$.
\end{remark}

\begin{remark}
\label{rem:gap}
Theorem~\ref{thm:optimality} identifies $\cM$ as the target set of optimal constrained solutions, but it does not by itself quantify what an algorithm can gain locally when it stops at a strict interior point. The next proposition makes this local trade-off precise: when the reward and cost gradients are positively aligned, a strictly feasible policy can convert part of its unused cost budget into a first-order reward improvement while remaining strictly feasible.
\end{remark}

\begin{proposition}
\label{prop:gap}
{\em (Local reward improvement from unused budget)} Suppose $\Jr,\Jc$ have $L$-Lipschitz gradients near $\bm{\theta}$ and $\gc(\bm{\theta})\neq\bm{0}$
(cf.\ Assumptions~\ref{asm:smooth} and~\ref{asm:regular}). Fix
$\beta\in(0,1)$ and consider the fractional normal step
\begin{equation}
  \Delta_\beta(\bm{\theta})
  \;=\;
  -\,\beta\,\frac{\delta(\bm{\theta})}{\norm{\gc(\bm{\theta})}^{2}}\,\gc(\bm{\theta}).
  \label{eq:frac-step}
\end{equation}
Then
\begin{align}
  \delta(\bm{\theta}+\Delta_\beta)
  &= (1-\beta)\,\delta(\bm{\theta}) + O\!\bigl(\delta(\bm{\theta})^{2}\bigr),
  \label{eq:delta-contract}\\
  \Jr(\pi_{\bm{\theta}+\Delta_\beta})-\Jr(\pith)
  &= -\,\beta\,\delta(\bm{\theta})\,
  \frac{\inner{\gr}{\gc}}{\norm{\gc}^{2}}
  + O\!\bigl(\delta(\bm{\theta})^{2}\bigr), \nonumber
  \label{eq:gap-approx}
\end{align}
where both remainders are bounded by
$\frac{L\beta^{2}}{2\norm{\gc}^{2}}\,\delta(\bm{\theta})^{2}$.
Consequently, if $\delta(\bm{\theta})<0$, $\inner{\gr}{\gc}>0$, and
\begin{equation}
  |\delta(\bm{\theta})|
  \;<\;
  \frac{2}{L\beta}\,
  \min\!\left\{\frac{(1-\beta)\norm{\gc}^{2}}{\beta},\;
  \inner{\gr}{\gc}\right\},
  \label{eq:small-delta}
\end{equation}
then $\bm{\theta}+\Delta_\beta$ is still strictly feasible and attains a
strictly larger reward, with first-order gain
$\beta\,|\delta(\bm{\theta})|\,\inner{\gr}{\gc}/\norm{\gc}^{2}$.
\end{proposition}

\noindent\emph{Proof.} By $L$-smoothness,
$\delta(\bm{\theta}+\Delta_\beta)=\delta(\bm{\theta})+\inner{\gc}{\Delta_\beta}+R_c$
with $|R_c|\leq\tfrac{L}{2}\norm{\Delta_\beta}^{2}$. Direct computation
gives $\inner{\gc}{\Delta_\beta}=-\beta\,\delta(\bm{\theta})$ and
$\norm{\Delta_\beta}=\beta|\delta(\bm{\theta})|/\norm{\gc}$,
hence $|R_c|\leq\frac{L\beta^{2}}{2\norm{\gc}^{2}}\delta(\bm{\theta})^{2}$. Identically,
$\Jr(\pi_{\bm{\theta}+\Delta_\beta})-\Jr(\pith)
=\inner{\gr}{\Delta_\beta}+R_r
=-\beta\,\delta(\bm{\theta})\inner{\gr}{\gc}/\norm{\gc}^{2}+R_r$
with $|R_r|\leq\frac{L\beta^{2}}{2\norm{\gc}^{2}}\delta(\bm{\theta})^{2}$. When $\delta(\bm{\theta})<0$, we have $(1-\beta)\delta(\bm{\theta})<0$ and
$\beta|\delta(\bm{\theta})|\inner{\gr}{\gc}/\norm{\gc}^{2}>0$, and both
strictly dominate the remainders exactly when \eqref{eq:small-delta} holds,
yielding strict feasibility and a strict reward increase. \hfill$\blacksquare$

Proposition~\ref{prop:gap} makes the intended claim precise: not every
interior policy incurs a reward loss; rather, under positive gradient
alignment, a fraction $\beta$ of the unused cost budget can be converted
into a first-order reward gain of
$\beta|\delta|\inner{\gr}{\gc}/\norm{\gc}^{2}$ while preserving strict
feasibility. This result motivates the constraint residual as a diagnostic
of boundary tracking. The residual alone is not a universal surrogate for
reward, because the local reward change also depends on the alignment and
magnitudes of the reward and cost gradients.

\section{Boundary-Seeking Policy Gradient}
\label{sec:algo}
Theorem~\ref{thm:optimality} and Corollary~\ref{cor:realizability} place the global constrained optimum on the active constraint set $\cM$ under their stated conditions; $\cM$ may also contain non-optimal points. A natural algorithmic strategy is therefore to improve reward along $\cM$ while regulating the iterates toward $\cM$. Motivated by this, we propose Boundary-Seeking Policy Gradient (BSPG) Algorithm (\cref{alg:bspg}) in this section, which implements this strategy through a gradient decomposition that separates the two goals into orthogonal components.

\subsection{Gradient Decomposition}
BSPG seeks to improve reward along the constraint manifold while steering the iterates toward it when necessary. These two roles naturally correspond to the tangential and normal directions of $\cM$, motivating the following decomposition. At any $\bm{\theta}$ with $\gc \neq \bm{0}$, the tangent space of the cost level set $\{\bm{\theta}':\Jc(\pi_{\bm{\theta}'})=\Jc(\pith)\}$ (which coincides with the tangent space of $\cM$ when $\bm{\theta}\in\cM$) is
$
  T_{\bm{\theta}}\cM = \{\bm{v} \in \RR^p : \inner{\bm{v}}{\gc} = 0\}.
$
The cost gradient $\bm g_c$ is normal to the level set $\{J_c=\text{const}\}$, which induces a natural splitting of any vector $\bm g\in \mathbb R^p$ into  the tangential part and the normal part
$\bm{g} = \bm{g}^{\perp} + \bm{g}^{\parallel}$ where:
\begin{align*}
  \bm{g}^{\perp} &= \bm{g} - \frac{\inner{\bm{g}}{\gc}}{\norm{\gc}^2}\gc
  \;\in\; T_{\bm{\theta}}\cM, 
  \\[-1mm]
  \bm{g}^{\parallel} &= \frac{\inner{\bm{g}}{\gc}}{\norm{\gc}^2}\gc
  \;\in\; \mathrm{span}(\gc). 
\end{align*}
By construction, a first-order update along $\bm g^\perp$ leaves the cost objective unchanged, since $\inner{\nablath \Jc}{\bm g^\perp} = \inner{\gc}{\bm g^\perp} = 0$. We therefore define the \emph{tangential reward gradient} as \vspace{-6pt}
\begin{align}
  \gperp = \gr - \frac{\inner{\gr}{\gc}}{\norm{\gc}^2}\gc.
  \label{eq:gperp}
\end{align}
\vspace{-12pt}

Hence, whenever $\gperp\neq\bm{0}$, it yields a first-order improvement in the reward objective while leaving the cost objective unchanged, irrespective of the angle $\theta_{rc}$.

\subsection{The BSPG Update}
The tangential reward gradient $\gperp$ captures reward-improving motion along the constraint manifold, but by construction it does not regulate the iterate's distance to $\cM$. To control this distance, we add a normal component that attracts the iterate toward the manifold from either side. Specifically, we define the boundary-attraction term as
\begin{equation}
  \gb = -\dhat \cdot \frac{\gc}{\norm{\gc} + \varepsilon},
  \label{eq:gb}
\end{equation}
\vspace{-12pt}

where $\dhat$ is the estimated constraint residual defined in Section~\ref{sec:estimation} and $\varepsilon > 0$ is a numerical stabilizer. Note that $\delta=\Jc-d$ is a residual in cost value, not a Euclidean distance in parameter space. The resulting BSPG update is
\begin{equation}
  \gup = \gperp + \eta \cdot \gb,
  \label{eq:bspg}
\end{equation}
\vspace{-16pt}

where $\eta > 0$ is the boundary-attraction coefficient, and the parameter update is
$\bm{\theta}_{t+1} = \bm{\theta}_t + \alpha_t \bm{g}_{\mathrm{up},t}.$ The update in \eqref{eq:bspg} combines a tangential component and a normal component. The term $\gperp$ performs reward ascent while preserving the cost objective to first order. The term $\gb$ provides signed attraction toward the boundary: when $\dhat > 0$, the iterate is infeasible and $\gb$ acts to decrease $\Jc$; when $\dhat < 0$, the iterate lies in the feasible interior and $\gb$ acts to increase $\Jc$. In both cases, the strength of this attraction scales with $|\dhat|$.

 Fig.~\ref{fig:grad_update} illustrates the resulting geometry against ESPO \cite{GuShiDin_24}: on the feasible side (Fig.~\ref{fig:grad_update}b), $\gb$ points outward toward $\cM$, reducing conservatism through a mechanism absent in ESPO (Fig.~\ref{fig:grad_update}a); on the infeasible side (Fig.~\ref{fig:grad_update}d), $\gb$ restores feasibility while $\gperp$ continues reward improvement along the boundary, whereas ESPO (Fig.~\ref{fig:grad_update}c) folds the cost correction into a single combined direction.
\begin{figure}[t]
    \centering
    \includegraphics[width=0.6\columnwidth]{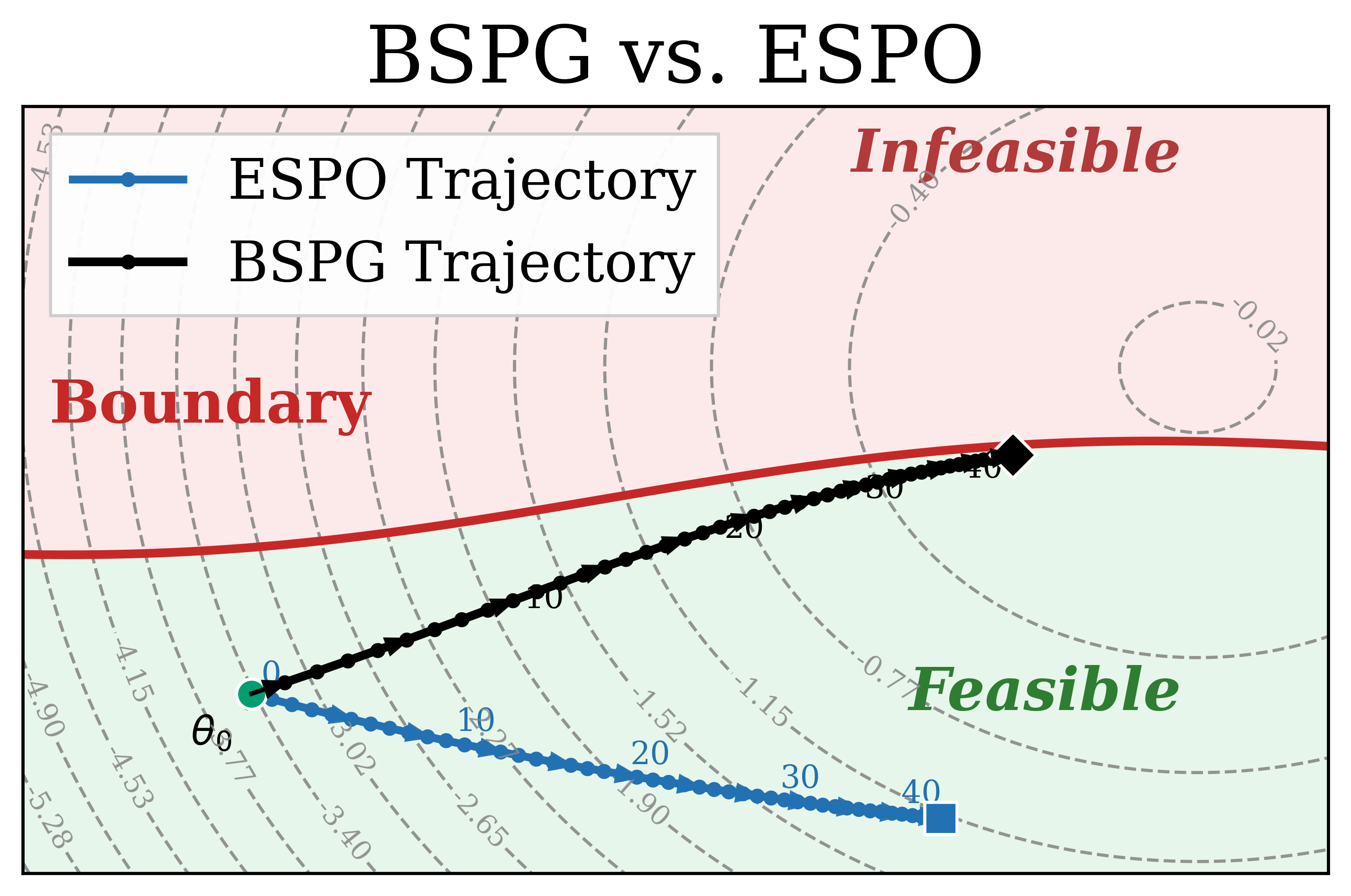}
    \caption{Update trajectories of BSPG and ESPO on a synthetic constrained problem; the example only visualizes the two update components.}
    \label{fig:update_trajectory}
    \vspace{-3ex}
\end{figure}
Fig.~\ref{fig:update_trajectory} traces both methods on a synthetic problem whose unconstrained reward maximizer lies beyond the boundary: from the same feasible start, ESPO remains interior while BSPG reaches the active boundary and then improves reward along it. 

\subsection{Implicit Lagrangian Interpretation} 
Although BSPG is derived from a geometric decomposition into tangential and normal directions, its update also admits an equivalent Lagrangian interpretation. In particular, the combined direction in \eqref{eq:bspg} can be rewritten as a gradient step on a Lagrangian objective with an automatically induced multiplier. Define the \emph{implicit multiplier} as:
\vspace{-4pt}
\begin{equation}
  \lamt = \frac{\inner{\gr}{\gc}}{\norm{\gc}^2}
  + \frac{\eta\dhat}{\norm{\gc} + \varepsilon}.
  \label{eq:lambda-t}
\end{equation}
\vspace{-8pt}

Substituting \eqref{eq:gperp} into \eqref{eq:bspg} and collecting the terms along $\gc$ yields
\vspace{-12pt}
\begin{align}
  \gup &= \gr - \frac{\inner{\gr}{\gc}}{\norm{\gc}^2}\gc
  - \eta\dhat \frac{\gc}{\norm{\gc} + \varepsilon} \nonumber\\[-1mm]
  &= \gr - \underbrace{\left(\frac{\inner{\gr}{\gc}}{\norm{\gc}^2}
  + \frac{\eta\dhat}{\norm{\gc} + \varepsilon}\right)}_{\displaystyle\lamt}\gc
  \label{eq:implicit}\\[-1mm]
  &= \nablath\!\left[\Jr(\pith) - \lambda \Jc(\pith)\right]\!\bigg|_{\lambda=\lamt},
  \label{eq:lagrangian-grad}
\end{align}
\vspace{-1em}

where the gradient in \eqref{eq:lagrangian-grad} is evaluated with $\lambda$ held fixed at $\lambda_t$.  Eq.~\eqref{eq:lagrangian-grad} shows that each BSPG step is equivalent to a gradient-ascent step on the Lagrangian $L(\theta,\lambda)=J_r(\pi_\theta)-\lambda J_c(\pi_\theta),$ with the multiplier $\lambda_t$ induced automatically by the local gradient geometry and the boundary signal.

The implicit multiplier in \eqref{eq:lambda-t} consists of two terms with distinct roles. The first term, $\inner{\gr}{\gc}/\norm{\gc}^2$, is the projection coefficient of the reward gradient onto the cost gradient. The second term, $\eta\hat\delta/\!(\norm{\gc}+\varepsilon)$, introduces a boundary-dependent correction that modulates the effective penalty according to the signed residual: in the feasible interior ($\hat\delta<0$) it lowers $\lamt$, tilting the update toward higher cost, and under positive gradient alignment, higher reward, until the residual closes. We emphasize that $\lamt$ is an algebraic coefficient rather than a learned dual variable, and it need not be nonnegative away from a KKT point. If the iterates converge to a stationary point on $\cM$, then $\hat\delta\to 0$ and $\lamt$ approaches the equality-stationarity coefficient $\lamstar$ of Theorem~\ref{thm:kkt}, which is a valid KKT multiplier exactly when the limit is also locally optimal over $\cF$.

\subsection{Boundary Distance Estimation}
\label{sec:estimation}
The BSPG update in \eqref{eq:bspg}, equivalently the implicit multiplier in \eqref{eq:lambda-t}, depends on the constraint residual $\delta(\bm{\theta}_t)=\Jc(\pi_{\bm{\theta}_t})-d$, which is not directly observable during training. We therefore estimate it by combining the learned cost-value network $V_c$ with a residual batch correction:
\begin{equation}
  \dhat_t = \underbrace{\frac{1}{N}\sum_{i=1}^N V_c(s_0^i)}_{\text{critic at }s_0}
  + \underbrace{\left(\bar{c}_{\mathrm{batch}} - \frac{1}{|\tau|}\sum_{(s,a) \in \tau} V_c(s)\right)}_{\text{residual bias correction}}
  - d.
  \label{eq:delta-est}
\end{equation}

where $N$ is the number of trajectories in the batch, $s_0^i$ are the initial states,
$\bar{c}_{\mathrm{batch}}$ is the mean discounted cost return from the batch, and
$|\tau|$ is the number of state--action pairs. The first term estimates the cost objective from the initial-state distribution through the critic; the second is a batch calibration term that vanishes when $V_c$ matches the empirical cost returns. Because the two averages are taken under different state distributions, no general unbiasedness or variance-reduction claim follows without additional assumptions on the critic and the sampling process; we use \eqref{eq:delta-est} as an implementation heuristic. The theory in Section~\ref{sec:theory} analyzes the exact residual $\dhat_t=\delta_t$ and does not depend on this estimator.

\subsection{Practical BSPG Implementation}
Algorithm~\ref{alg:bspg} presents the practical implementation of BSPG built upon PPO\cite{SchWolDha_17}. Lines 5--7 estimate the boundary distance(Eq.\eqref{eq:delta-est}), while lines 11--16 compute the shaped update direction(Eq.\eqref{eq:gperp},\eqref{eq:bspg}); the remaining steps follow the standard PPO procedure. Relative to PPO, BSPG introduces only two additional hyperparameters: the boundary-attraction coefficient $\eta$ and the clipping threshold $\delta_{\max}$ for $\hat\delta$.

\begin{algorithm}[t]
\caption{BSPG}
\label{alg:bspg}
\begin{algorithmic}[1]
\State \textbf{Input:} Policy $\pith$, value networks $V$, $V_c$, threshold $d$, coefficient $\eta$, stabilizer $\varepsilon$, clip range $\varepsilon_{\mathrm{clip}}$, residual clip $\delta_{\max}$, step size $\alpha$
\While{not converged}
  \State Collect rollout $\tau \sim \pith$ for $T$ steps
  \State Compute GAE reward advantages $\hat{A}$ using $V$; cost advantages $\hat{A}_c$ using $V_c$
  \State \textcolor{mygreen}{\texttt{//Boundary estimation (Eq.~\ref{eq:delta-est})}}
  \State $\dhat \! \leftarrow \! \frac{1}{N}\sum_i V_c(s_0^i) \!+\! \bigl(\bar{c}_{\mathrm{batch}} \!- \!\frac{1}{|\tau|}\sum_{s\in\tau}V_c(s)\bigr) - d$
  \State $\dhat \leftarrow \clip\!\left(\dhat,\, -\delta_{\max},\, \delta_{\max}\right)$
  \State \textcolor{mygreen}{\texttt{//Surrogate gradients (PPO clipped objectives)}}
    \State $L_r(\bm{\theta}) \leftarrow \text{PPO-clip surrogate for reward using } \hat{A}$
  \State $L_c(\bm{\theta}) \leftarrow \text{PPO-clip surrogate for cost using } \hat{A}_c$
  \State $\gr \leftarrow \nablath L_r(\bm{\theta})$;\quad $\gc \leftarrow \nablath L_c(\bm{\theta})$
  \State Normalize: $\hat{\gr} \leftarrow \gr/(\norm{\gr}+\varepsilon)$;\quad $\hat{\gc} \leftarrow \gc/(\norm{\gc}+\varepsilon)$
  \State \textcolor{mygreen}{\texttt{//Gradient decomposition}}
  \State $\gperp \leftarrow \hat{\gr} - \dfrac{\inner{\hat{\gr}}{\hat{\gc}}}{\norm{\hat{\gc}}^2 + \varepsilon}\hat{\gc}$
  \State \textcolor{mygreen}{\texttt{// Boundary attraction}}
  \State $\gb \leftarrow -\dhat \cdot \hat{\gc} / (\norm{\hat{\gc}} + \varepsilon)$
  \State \textcolor{mygreen}{\texttt{// Shaped gradient update}}
  \State $\bm{\theta} \leftarrow \bm{\theta} + \alpha\,(\gperp + \eta\,\gb)$
    \State Update $V$ and $V_c$ via MSE on GAE return targets
\EndWhile
\end{algorithmic}
\end{algorithm}

\section{Theoretical Analysis}
\label{sec:theory}

We now establish the convergence properties of the ideal BSPG update with exact gradients and $\dhat_t=\delta_t$; the statements below do not claim convergence of the stochastic normalized PPO implementation. Three results are proved: the constraint residual converges to zero from either side (with a finite-horizon rate), the tangential component is a reward-ascent direction on the boundary, and any convergent run is characterized at its limit. The proofs of Theorems~\ref{thm:boundary}--\ref{thm:kkt} are provided in Appendices~\ref{sec:proof1}--\ref{sec:proof3}. We begin by collecting the regularity assumptions.

\begin{assumption}[Smoothness and bounded gradients]
\label{asm:smooth}
There is a compact set $\Theta$ containing all exact BSPG iterates and the line segments between consecutive iterates. On $\Theta$, $\Jr$ and $\Jc$ have $L$-Lipschitz gradients, and $\norm{\gr},\norm{\gc}\leq G$.
\end{assumption}

\begin{assumption}[Constraint regularity]
\label{asm:regular}
There are constants $\mu>0$ and $D<\infty$ such that
$\norm{\gc(\bm{\theta})}\geq\mu$ and $|\delta(\bm{\theta})|\leq D$
for all $\bm{\theta}\in\Theta$.
\end{assumption}

\begin{assumption}[Diminishing step sizes]
\label{asm:stepsize}
The step size sequence $\{\alpha_t\}_{t \geq 0}$ satisfies:
$
  \sum_{t=0}^{\infty} \alpha_t = \infty ,$ and $\sum_{t=0}^{\infty} \alpha_t^2 < \infty.
$
\end{assumption}

Assumptions~\ref{asm:smooth} and~\ref{asm:stepsize} are standard in first-order constrained policy optimization~\cite{XuLiaLan_21}; if $0\leq c(s,a)\leq c_{\max}$ and $0\leq d\leq c_{\max}/(1-\gamma)$, the residual bound holds automatically with $D=\max\{d,\,c_{\max}/(1-\gamma)-d\}$.
Assumption~\ref{asm:regular} is a nondegeneracy (constraint-qualification) condition excluding points where the cost gradient vanishes; it makes the local cost level sets regular and ensures the decomposition \eqref{eq:gperp} is well defined at every iterate.

We first show that BSPG drives the constraint residual $\Jc(\pi_{\bm{\theta}_t})-d$ to zero.

\begin{theorem}[Boundary Convergence]
\label{thm:boundary}
Under Assumptions~\ref{asm:smooth}--\ref{asm:stepsize}, let
$\delta_t\triangleq\delta(\bm{\theta}_t)=\Jc(\pi_{\bm{\theta}_t})-d$ denote the
signed constraint violation at iteration $t$.  The sequence
$\{\delta_t\}$ generated by exact-gradient BSPG satisfies:
\begin{equation}
  V(\bm{\theta}_{t+1}) \leq V(\bm{\theta}_t)
  - \mu'\alpha_t\,\delta_t^2 + C_R\alpha_t^2,
  \label{eq:lyapunov-main}
\end{equation}
where $V(\bm{\theta}) = \tfrac{1}{2}\delta(\bm{\theta})^2$,
$\mu' = \frac{\eta\mu^2}{G+\varepsilon}$,
and $C_R = (G^2+LD)(G^2+\eta^2 D^2)$.
Consequently:
\begin{equation}
  \sum_{t=0}^{\infty} \alpha_t\,\delta_t^2 < \infty
  \qquad\text{and}\qquad \delta_t \to 0.
  \label{eq:boundary-conv}
\end{equation}
Thus, $\Jc(\pi_{\bm{\theta}_t})\to d$ from either side of the boundary.
\end{theorem}

\begin{corollary}[Finite-horizon residual bound]
\label{cor:finite-residual}
Let $A_T=\sum_{t=0}^{T-1}\alpha_t$ and draw $\tau\in\{0,\ldots,T-1\}$ with $\Pr(\tau=t)=\alpha_t/A_T$. Then
\begin{equation}
  \EE[\delta_\tau^2]
  \;\leq\;
  \frac{V_0+C_R\sum_{t=0}^{T-1}\alpha_t^2}{\mu' A_T}.
  \label{eq:finite-residual}
\end{equation}
In particular, for constant $\alpha_t=\bar{\alpha}/\sqrt{T}$ over a fixed horizon $T$, the bound is $O(1/\sqrt{T})$.
\end{corollary}
The corollary follows by summing \eqref{eq:lyapunov-main} over $t<T$ and dividing by $\mu' A_T$; see Appendix~\ref{sec:proof1}.

The Lyapunov function $V=\tfrac{1}{2}\delta^2$ is symmetric in the sign of the residual, so the ideal normal component regulates the target cost level from either side. This differs from a standard projected dual update: a negative residual can decrease the dual variable only until it reaches zero, after which the boundary signal is lost, whereas BSPG continues to apply a residual-proportional normal correction toward the prescribed cost level.

It is worth noting that the Lyapunov function $V = \frac{1}{2}\delta^2$ is symmetric in the sign of $\delta$, so the boundary-seeking mechanism operates identically whether the current iterate is feasible or infeasible.
This contrasts with Lagrangian methods, where the multiplier update acts only on positive violations.

The boundary-attraction term drives $\delta\to 0$; however, it does not by itself guarantee reward improvement.
The following theorem shows that the tangential component $\gperp$ is a valid ascent direction for $J_r$.

\begin{theorem}[Reward Ascent]
\label{thm:ascent}
For any $\bm{\theta}$ with $\gc \neq \bm{0}$:
\begin{equation}
  \inner{\gr}{\gperp} = \norm{\gperp}^2 \geq 0.
\end{equation}
Equality holds iff $\gr\parallel\gc$.  On the manifold $\cM$ (where
$\delta=0$ and $\gup=\gperp$), the BSPG update is a strict ascent direction for $\Jr$
unless $\gr\parallel\gc$, i.e., unless $\bm{\theta}$ is a first-order stationary point of the equality-constrained problem $\max\Jr$ subject to $\Jc=d$.
\end{theorem}

We next characterize convergent runs of the exact update.
\begin{theorem}[Convergent-limit characterization]
\label{thm:kkt}
Under Assumptions~\ref{asm:smooth}--\ref{asm:stepsize}, suppose the
exact-gradient BSPG iterates converge: $\bm{\theta}_t\to\bm{\theta}^*$.
Then $\bm{\theta}^*\in\cM$ and
\begin{equation}
  \gr(\bm{\theta}^*)=\lamstar\gc(\bm{\theta}^*),
  \qquad
  \lamstar=\frac{\inner{\gr(\bm{\theta}^*)}{\gc(\bm{\theta}^*)}}{\norm{\gc(\bm{\theta}^*)}^2},
  \label{eq:limit-characterization}
\end{equation}
i.e., $\bm{\theta}^*$ is a first-order stationary point for maximizing $\Jr$ on the active constraint set, and complementary slackness $\lamstar(\Jc(\pithstar)-d)=0$ holds. If, in addition, $\bm{\theta}^*$ is a local maximizer of $\Jr$ on $\cF$, then $\lamstar\geq 0$ and $\bm{\theta}^*$ satisfies the KKT conditions of \eqref{eq:cmdp}.
\end{theorem}

\begin{remark}[Limit-set extension]
\label{rem:limit-set}
The sequential hypothesis $\bm{\theta}_t\to\bm{\theta}^*$ admits a fully elementary, self-contained proof (Appendix~\ref{sec:proof3}). Under the same assumptions, a limit-set version --- every limit point of $\{\bm{\theta}_t\}$ satisfies \eqref{eq:limit-characterization} --- can be obtained through the ODE method and the chain-transitive form of LaSalle's invariance principle \cite{Bor_09}; we state the sequential version to keep the argument self-contained.
\end{remark}

\begin{remark}[Comparison with PCRPO]
\label{cor:pcrpo}
PCRPO \cite{GuSelDin_24} establishes convergence to a KKT point $\bm{\theta}^\dagger$ satisfying $\Jc(\pi_{\bm{\theta}^\dagger})\le d$, without further characterizing the constraint value at the limit. Under exact gradients, Theorem~\ref{thm:boundary} drives the residual to zero, and Theorem~\ref{thm:kkt} guarantees that the constraint is active, $\Jc(\pithstar)=d$, at the limit of any convergent run. Moreover, whenever a method stops at an interior point $\bm{\theta}^\dagger$ with residual $\delta^\dagger=\Jc(\pi_{\bm{\theta}^\dagger})-d<0$ and $\inner{\gr}{\gc}>0$, Proposition~\ref{prop:gap} shows that a fraction $\beta$ of the unused budget is locally convertible into a first-order reward gain of $\beta|\delta^\dagger|\inner{\gr}{\gc}/\norm{\gc}^2$ while preserving strict feasibility.
\end{remark}

\begin{remark}[Scope of the theory]
\label{rem:scope}
Theorem~\ref{thm:optimality} and Corollary~\ref{cor:realizability} identify the global constrained target; Theorems~\ref{thm:boundary}--\ref{thm:kkt} show that the ideal parameter update regulates this target cost level and characterize convergent runs. They do not establish global reward convergence for an arbitrary nonconvex policy class (which would require additional landscape structure or a separate tabular analysis), nor sample complexity or convergence guarantees for the stochastic normalized PPO implementation.
\end{remark}
\section{Experiments} \label{sec:experiments}
\noindent \textbf{Evaluation Setups.} We evaluate our method on \textbf{SafetyPointGoal1-v0} using the OmniSafe\cite{JiZhoZha_24} framework. This task provides a standard constrained navigation setting for studying the reward--cost tradeoff in safe policy optimization.

\noindent \textbf{Baselines.}
We compare BSPG against two closely related constrained RL baselines: \textbf{CRPO}\cite{XuLiaLan_21} and \textbf{ESPO}\cite{GuShiDin_24}. CRPO is a constraint-driven policy optimization method that switches update behavior according to constraint satisfaction, emphasizing feasibility control. ESPO is a recent first-order safe policy optimization method that directly balances reward improvement and cost reduction through gradient-based policy updates. 

\noindent \textbf{Evaluation Metrics.} We report three primary metrics throughout training: (i) \textit{average episodic reward}, where higher values indicate better task performance; (ii) \textit{average episodic cost}, where lower values indicate better constraint control; and (iii) \textit{boundary proximity}, the empirical residual magnitude $|\widehat{C} - d|$, where $\widehat{C}$ is the evaluation cost statistic; smaller values mean that the reported policy uses the prescribed budget more fully. The third metric is our primary diagnostic and, to our knowledge, has not been explicitly reported in prior Safe-RL work. By Theorem~\ref{thm:optimality}, smaller boundary proximity indicates tighter tracking of the active constraint boundary; by Proposition~\ref{prop:gap}, under positive reward--cost gradient alignment, the remaining residual corresponds to locally recoverable reward. The residual alone is not a universal surrogate for reward, since the local reward change also depends on the gradient alignment and magnitudes.

As shown in Fig.~\ref{fig:reward_cost}, BSPG attains the highest episodic reward on SafetyPointGoal1-v0. At the same time, BSPG maintains episodic cost near the constraint limit, while both baselines remain more conservative and settle well inside the feasible region. This behavior is most clearly illustrated in the boundary-proximity plot, where BSPG remains consistently closest to zero, indicating tighter tracking of the active constraint boundary. Taken together, these results support the proposed boundary-regulation mechanism in the evaluated task; they do not by themselves imply a general performance ordering across environments.
\begin{figure}[t]
    \centering
    \includegraphics[width=\columnwidth]{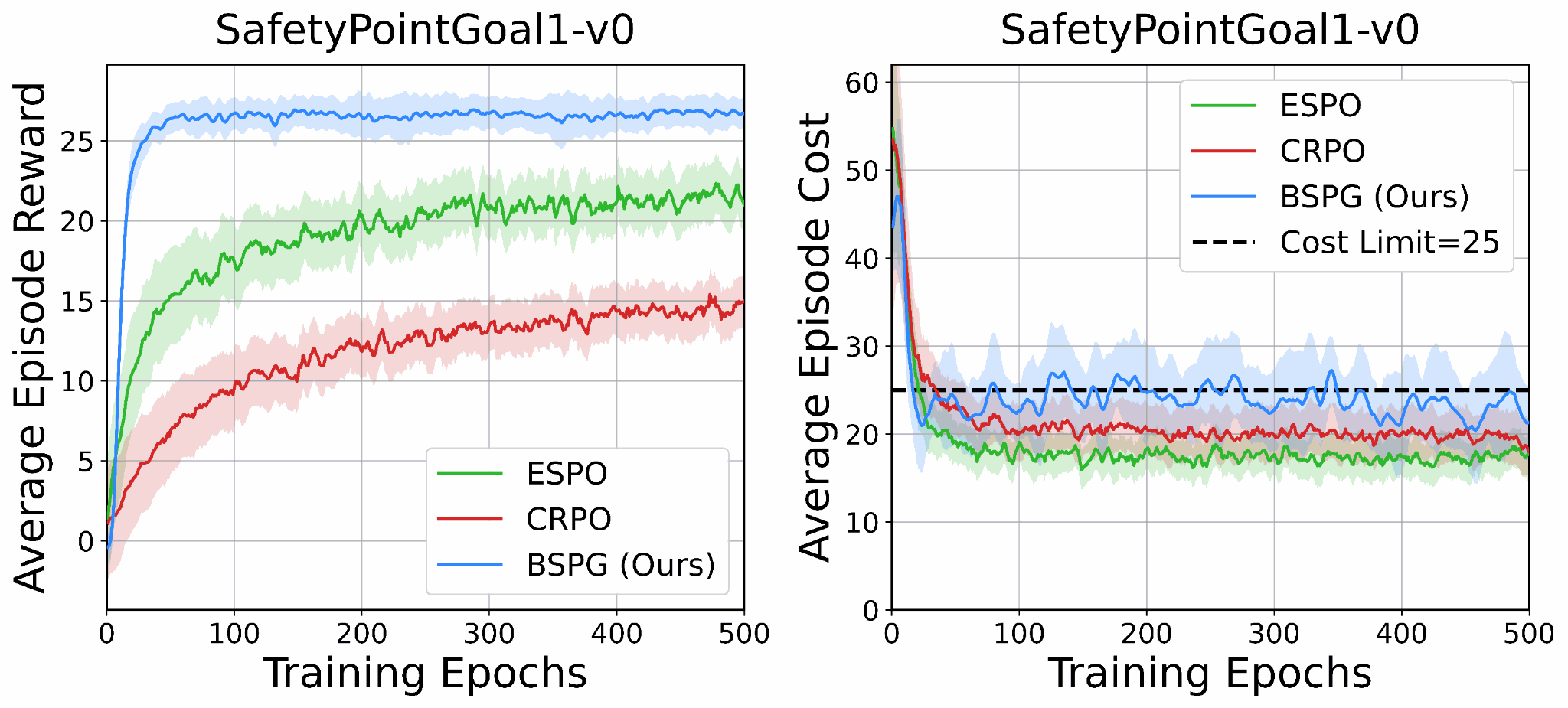}
    \vspace{-1.5em}
    \caption{Results on SafetyPointGoal1-v0.}
    \label{fig:reward_cost}
    \vspace{-3ex}
\end{figure}

\section{Conclusion}
In this paper, we proposed Boundary-Seeking Policy Gradient (BSPG), a safe reinforcement learning method that explicitly drives policy updates toward the active constraint boundary while improving reward along it. By combining a tangential reward term with a signed boundary-attraction term, BSPG encourages boundary-aligned optimization rather than conservative interior solutions. On the theory side, occupancy-measure analysis identifies the active boundary as the global constrained target when no unconstrained reward maximizer is feasible, and the realizability corollary transfers this target to a policy class attaining the exact constrained value. For the ideal exact-gradient update, BSPG drives the signed residual to zero from either side with a finite-horizon $O(1/\sqrt{T})$ bound, and any convergent run reaches a stationary point on the active constraint set, satisfying the KKT conditions when the limit is locally optimal over the feasible set. Empirically, BSPG attains a strong reward--safety trade-off in the evaluated task by making fuller use of the safety budget.
\section*{Acknowledgement}
This work is partially supported by NSF ECCS Award \#2534263.

\bibliographystyle{IEEEtran}

\appendix

\subsection{Proof of Theorem~\ref{thm:optimality}}
\label{sec:proof0}
For a stationary policy $\pi$, define its (normalized) occupancy measure
$\rho_\pi(s,a)=(1-\gamma)\sum_{t=0}^\infty \gamma^t \Pr_{\pi,\mu_0}(s_t=s,a_t=a)$.
The set $\Lambda=\{\rho_\pi:\pi\in\Pi\}$ is a compact convex polytope, and
$\Jr,\Jc$ are linear functionals of $\rho$:
$J_k(\pi)=\frac{1}{1-\gamma}\sum_{s,a}\rho_\pi(s,a)\,k(s,a)$ for
$k\in\{r,c\}$ \cite{Alt_99}; hence the maxima below are attained.

Suppose, for contradiction, that some constrained-optimal $\pi^\star$
satisfies $\Jc(\pi^\star)<d$, and write $\rho^\star=\rho_{\pi^\star}$.
Fix any unconstrained maximizer $\pi_r^\star$ with occupancy measure
$\rho_r$; by hypothesis $\Jc(\rho_r)>d$. If $\Jr(\rho_r)=\Jr(\rho^\star)$,
then $\pi^\star$ would itself be a feasible unconstrained maximizer,
contradicting the hypothesis; hence $\Jr(\rho_r)>\Jr(\rho^\star)$.
For $\lambda\in(0,1]$ let $\rho_\lambda=(1-\lambda)\rho^\star+\lambda\rho_r
\in\Lambda$ (convexity). By linearity,
$\Jr(\rho_\lambda)=(1-\lambda)\Jr(\rho^\star)+\lambda\Jr(\rho_r)
>\Jr(\rho^\star)$, and
$\Jc(\rho_\lambda)=(1-\lambda)\Jc(\rho^\star)+\lambda\Jc(\rho_r)\le d$
for every
$\lambda\le\bigl(d-\Jc(\rho^\star)\bigr)/\bigl(\Jc(\rho_r)-\Jc(\rho^\star)\bigr)$,
a strictly positive threshold. Thus, for small $\lambda>0$, the policy
induced by $\rho_\lambda$ is feasible with strictly larger reward than
$\pi^\star,$ a contradiction. Hence $\Jc(\pi^\star)=d$.
\hfill$\blacksquare$

\subsection{Proof of Theorem~\ref{thm:boundary}}
\label{sec:proof1}
Define $V(\bm{\theta})=\tfrac{1}{2}\delta(\bm{\theta})^2$.
By the chain rule, $\nablath V=\delta\gc$.
The function $V$ is $L_V$-smooth with $L_V=G^2+LD$:
for any $\bm{\theta},\bm{\theta}'\in\RR^p$,
\begin{align}
  \norm{\nablath V(\bm{\theta}')\!-\!\nablath V(\bm{\theta})}
  \leq & \bigl(|\delta'-\delta|\!\cdot\!\norm{\gc'}+|\delta|\!\cdot\!\norm{\gc'-\gc}\bigr) \nonumber \\
  \leq & (G^2+LD)\norm{\bm{\theta}'-\bm{\theta}},
  \label{eq:LV-smooth}
\end{align}
using $|\delta'-\delta|\leq G\norm{\bm{\theta}'-\bm{\theta}}$ (mean value theorem and Assumption~\ref{asm:smooth}),
$\norm{\gc}\leq G$ (Assumption~\ref{asm:smooth}), $\norm{\gc'-\gc}\leq L\norm{\bm{\theta}'-\bm{\theta}}$ (Assumption~\ref{asm:smooth}),
and $|\delta|\leq D$ (Assumption~\ref{asm:regular}); the segment condition in Assumption~\ref{asm:smooth} ensures these bounds hold along the update segments.
By the descent lemma:
\vspace{-6pt}
\begin{equation}
  V_{t+1} \leq V_t + \alpha_t\delta_t\inner{\gc^t}{\gup^t}
  + \frac{L_V\alpha_t^2}{2}\norm{\gup^t}^2.
  \label{eq:V-descent}
\end{equation}
\vspace{-6pt}

Since $\gup=\gperp+\eta\gb$ with exact gradients:

\emph{Tangential term.}
$\inner{\gc}{\gperp}
=\inner{\gc}{\gr-\frac{\inner{\gr}{\gc}}{\norm{\gc}^2}\gc}
=\inner{\gc}{\gr}-\inner{\gr}{\gc}=0.$

\emph{Boundary term.}
$\gb=-\delta_t\gc/(\norm{\gc}+\varepsilon)$ (exact $\delta_t$, exact $\gc$), so:
\begin{equation}
  \delta_t\inner{\gc}{\gb}
  =-\delta_t^2\frac{\norm{\gc}^2}{\norm{\gc}+\varepsilon}
  \leq -\delta_t^2\frac{\mu^2}{G+\varepsilon},
  \label{eq:boundary-drift}
\end{equation}
where we used $\norm{\gc}\geq\mu$ (Assumption~\ref{asm:regular})
and $\norm{\gc}\leq G$ (Assumption~\ref{asm:smooth}).
Therefore:
\vspace{-6pt}
\begin{equation}
  \delta_t\inner{\gc^t}{\gup^t}
  =\eta\delta_t\inner{\gc^t}{\gb^t}
  \leq -\underbrace{\frac{\eta\mu^2}{G+\varepsilon}}_{\mu'}\delta_t^2.
  \label{eq:drift}
\end{equation}
\vspace{-6pt}

We know that
$\norm{\gup}^2\leq 2\norm{\gperp}^2+2\eta^2\norm{\gb}^2
\leq 2G^2+2\eta^2 D^2 \triangleq C_1$,
since $\norm{\gperp}\leq\norm{\gr}\leq G$ and
$\norm{\gb}=|\delta_t|\norm{\gc}/(\norm{\gc}+\varepsilon)\leq|\delta_t|\leq D$.

Substituting the two bounds above into \eqref{eq:V-descent}:
\begin{equation}
  V_{t+1}\leq V_t - \mu'\alpha_t\delta_t^2 + C_R\alpha_t^2,
  \quad C_R=\tfrac{L_V C_1}{2}.
\end{equation}
This is \eqref{eq:lyapunov-main}.

Finally, summing from $0$ to $T-1$:
$\mu'\sum_{t=0}^{T-1}\alpha_t\delta_t^2\leq V_0+C_R\sum_{t=0}^\infty\alpha_t^2<\infty$.
Letting $T\to\infty$ gives $\sum\alpha_t\delta_t^2<\infty$.

To show $\delta_t\to 0$: the recursion $V_{t+1}\leq(1-2\mu'\alpha_t)V_t+C_R\alpha_t^2$
is a Robbins--Monro recursion with summable perturbation. Since $\alpha_t\to 0$,
there exists $N$ with $0\leq 1-2\mu'\alpha_t\leq 1$ for all $t\geq N$; iterating
from $N$ and using $1-x\leq e^{-x}$,
$\prod_{s=N}^{t}(1-2\mu'\alpha_s)\leq\exp(-2\mu'\sum_{s=N}^t\alpha_s)\to 0$,
and the convolution with $\alpha_s^2$ vanishes (standard argument), so
$V_t\to 0$ and $\delta_t\to 0$.

\smallskip
\noindent\emph{Proof of Corollary~\ref{cor:finite-residual}.}
Summing \eqref{eq:lyapunov-main} over $t=0,\ldots,T-1$ gives
$\mu'\sum_{t=0}^{T-1}\alpha_t\delta_t^2\leq V_0+C_R\sum_{t=0}^{T-1}\alpha_t^2$.
Since $\EE[\delta_\tau^2]=\frac{1}{A_T}\sum_{t=0}^{T-1}\alpha_t\delta_t^2$, dividing by $\mu' A_T$ yields \eqref{eq:finite-residual}. With $\alpha_t=\bar{\alpha}/\sqrt{T}$, $A_T=\bar{\alpha}\sqrt{T}$ and $\sum_t\alpha_t^2=\bar{\alpha}^2$, giving the $O(1/\sqrt{T})$ rate. \hfill$\blacksquare$

\subsection{Proof of Theorem~\ref{thm:ascent}}
\label{sec:proof2}
Compute:
$\inner{\gr}{\gperp}
=\inner{\gr}{\gr-\frac{\inner{\gr}{\gc}}{\norm{\gc}^2}\gc}
=\norm{\gr}^2-\frac{\inner{\gr}{\gc}^2}{\norm{\gc}^2}$.
Also:
$\norm{\gperp}^2
=\norm{\gr}^2-2\frac{\inner{\gr}{\gc}^2}{\norm{\gc}^2}
+\frac{\inner{\gr}{\gc}^2}{\norm{\gc}^2}
=\norm{\gr}^2-\frac{\inner{\gr}{\gc}^2}{\norm{\gc}^2}$.
Both expressions are equal and non-negative by Cauchy--Schwarz. Equality holds iff $\Vert \gperp\Vert=0$ iff  $\gr\parallel\gc,$ confirming that $\gr=\alpha\gc$ for some scalar $\alpha$.

\subsection{Proof of Theorem~\ref{thm:kkt}}
\label{sec:proof3}
By Theorem~\ref{thm:boundary}, $\delta_t\to 0$; since $\bm{\theta}_t\to\bm{\theta}^*$ and $\delta$ is continuous, $\delta(\bm{\theta}^*)=0$, i.e., $\bm{\theta}^*\in\cM$.

\textbf{Stationarity.} Suppose, for contradiction, that $\bm{u}\triangleq\gperp(\bm{\theta}^*)\neq\bm{0}$. Since $\norm{\gc}\geq\mu$ on $\Theta$ (Assumption~\ref{asm:regular}), the map $\bm{\theta}\mapsto\gperp(\bm{\theta})$ is continuous at $\bm{\theta}^*$, and the normal component vanishes in the limit, $\norm{\eta\gb(\bm{\theta}_t)}\leq\eta|\delta_t|\to 0$. Hence $\gup(\bm{\theta}_t)=\gperp(\bm{\theta}_t)+\eta\gb(\bm{\theta}_t)\to\bm{u}$, and there exists $N$ such that
$
  \inner{\bm{u}}{\gup(\bm{\theta}_t)}\;\geq\;\tfrac{1}{2}\norm{\bm{u}}^2,
  \qquad\forall t\geq N.
$
Projecting the recursion $\bm{\theta}_{t+1}-\bm{\theta}_t=\alpha_t\gup(\bm{\theta}_t)$ onto $\bm{u}$ and summing from $N$ to $T-1$:
$  \inner{\bm{u}}{\bm{\theta}_T-\bm{\theta}_N}
  =\sum_{t=N}^{T-1}\alpha_t\inner{\bm{u}}{\gup(\bm{\theta}_t)}
  \;\geq\;\tfrac{1}{2}\norm{\bm{u}}^2\sum_{t=N}^{T-1}\alpha_t
  \;\xrightarrow[T\to\infty]{}\;\infty,
$
since $\sum_t\alpha_t=\infty$ (Assumption~\ref{asm:stepsize}). This contradicts $\bm{\theta}_T\to\bm{\theta}^*$, whose left side converges to the finite value $\inner{\bm{u}}{\bm{\theta}^*-\bm{\theta}_N}$. Therefore $\gperp(\bm{\theta}^*)=\bm{0}$, and since $\norm{\gc(\bm{\theta}^*)}\geq\mu>0$ we may solve for $\gr$:
$\gr(\bm{\theta}^*)=\frac{\inner{\gr}{\gc}}{\norm{\gc}^2}\gc=:\lamstar\gc(\bm{\theta}^*)$.
Complementary slackness $\lamstar(\Jc(\pithstar)-d)=0$ is immediate from $\bm{\theta}^*\in\cM$.

\textbf{Dual feasibility under local optimality.}
Suppose $\bm{\theta}^*$ is a local maximizer of $\Jr$ on $\cF$ and, for contradiction, $\lamstar<0$.
From stationarity, $\gr=\lamstar\gc$.  Consider the direction $\bm{v}=-\gc$:
$\inner{\gr}{\bm{v}}=-\lamstar\norm{\gc}^2>0$ (reward increases to first order),
$\inner{\gc}{\bm{v}}=-\norm{\gc}^2<0$ (cost strictly decreases, so the point enters the feasible interior since $\Jc(\pithstar)=d$).
Hence, for sufficiently small $\varepsilon_0>0$, the point $\bm{\theta}^*-\varepsilon_0\gc\in\cF$ attains strictly higher reward,
contradicting local maximality. Therefore $\lamstar\geq 0$, and together with stationarity, primal feasibility, and complementary slackness, $\bm{\theta}^*$ satisfies the KKT conditions of \eqref{eq:cmdp}. \hfill$\blacksquare$

\end{document}